\documentclass{article}

\usepackage[preprint]{neurips_2025}

\usepackage[utf8]{inputenc} 
\usepackage[T1]{fontenc}    
\usepackage{hyperref}       
\usepackage{url}            
\usepackage{booktabs}       
\usepackage{amsfonts}       
\usepackage{nicefrac}       
\usepackage{microtype}      
\usepackage{xcolor}         
\usepackage{pgfmath}

\usepackage{amsmath}
\usepackage{mathtools}
\usepackage{natbib}
\usepackage{amssymb}
\usepackage{amsthm}
\usepackage{mathtools}
\usepackage{todonotes}
\usepackage{algorithm, algorithmic}
\newtheorem{lemma}{Lemma}
\newtheorem{theorem}{Theorem}
\newtheorem{defn}{Definition}
\newtheorem{corollary}{Corollary}

\newcommand{\indic}[1]{\mathbb{I}\left\{#1\right\}}

\newcommand{\D}{\mathcal{D}}
\renewcommand{\H}{\mathcal{H}}

\newcommand{\X}{\mathcal{X}}
\newcommand{\F}{\mathcal{F}}

\newcommand{\E}{\mathbb{E}}
\newcommand{\R}{\mathbb{R}}
\newcommand{\z}{\mathbf{z}}

\newcommand{\w}{\mathbf{w}}
\newcommand{\x}{\mathbf{x}}

\newcommand{\srarg}[1]{SR_{#1}}

\newcommand{\jsr}{SR_\lambda}
\newcommand{\sr}{SR_\lambda(Z)}
\newcommand{\expo}[1]{\exp \left( #1 \right)}
\newcommand{\prob}[1]{\mathbb{P} \left[ #1 \right]}
\newcommand{\tr}{^\mathsf{\scriptscriptstyle T}}
\newcommand{\argmin}{\text{argmin}}

\newcommand\numberthis{\addtocounter{equation}{1}\tag{\theequation}}

\newcommand{\tl}{t_{\textrm{L}}}
\newcommand{\tu}{t_{\textrm{U}}}
\newcommand{\srest}{\srarg{n}(\z)}
\def\argmin{\mathop{\rm arg\,min}}

\usepackage{cleveref}
\crefname{subsection}{subsection}{subsections}
\crefname{lemma}{lemma}{lemmas}
\crefname{property}{property}{properties}
\crefname{table}{table}{tables}
\crefname{assumption}{assumption}{assumptions}
\crefformat{equation}{(#2#1#3)}

\newtheorem{assumption}{}

\newtheorem{remark}{Remark}

\title{Optimizing Shortfall Risk Metric for Learning Regression Models }
\author{  Harish G. Ramaswamy\\
	Indian Institute of Technology Madras,\\
	Chennai, India\\
	\texttt{hariguru@dsai.iitm.ac.in} \\
	\And
	Prashanth L. A.\thanks{A portion of this work was done when Prashanth L. A. was at IIT Bombay.}\\
	Indian Institute of Technology Madras,\\
	Chennai, India\\
	\texttt{prashla@cse.iitm.ac.in} \\
}

\begin{document}

\maketitle

\begin{abstract}
    We consider the problem of estimating and optimizing utility-based shortfall risk (UBSR) of a loss, say $(Y - \hat Y)^2$, in the context of a regression problem. Empirical risk minimization with a UBSR objective is challenging since UBSR is a non-linear function of the underlying distribution.   We first derive a concentration bound for UBSR estimation using independent and identically distributed (i.i.d.) samples. We then frame the UBSR optimization problem as minimization of a pseudo-linear function in the space of achievable distributions $\D$ of the loss $(Y- \hat Y)^2$. We construct a gradient oracle for the UBSR objective and a linear minimization oracle (LMO) for the set $\D$. Using these oracles, we devise a bisection-type algorithm, and establish convergence to the UBSR-optimal solution.
\end{abstract}

\section{Introduction}
Expected risk minimization is the dominant objective in supervised learning problems. 
In several applications, an expected value objective may not be appealing as it may be necessary to mitigate risks. The term `risk' is used in two contexts. The traditional machine learning usage of `risk' is in the sense of minimizing the prediction error, while a risk measure is a functional, which is usually non-linear,  of the underlying distribution. In a regression problem, the distribution is of the loss random variable, say $(Y-\hat Y)^2$. 

We outline the risk measures that have been proposed in the literature. 
Value-at-risk (VaR), which in financial parlance is the maximum loss that a financial position can incur with a given confidence level. Sticking to financial jargon, VaR is not desirable as a risk measure since it is not sub-additive, which implies diversification could indeed increase the risk.
Conditional value-at-risk (CVaR) \cite{rockafellar2000optimization} is a popular risk measure and is the conditional expectation of a random variable (r.v.) beyond the VaR. CVaR is sub-additive and belongs to the class of coherent risk measures \cite{artzner1999coherent}, which requires a risk measure be monotone, positively homogeneous, and translation invariant in addition to sub-additivity.

Convex risk measures \cite{follmer2002convex} subsume coherence, as positive homogeneity and sub-additivity imply convexity. Utility-based shortfall risk (UBSR) is a popular convex risk measure that is closely related to CVaR.
A natural question is why employ UBSR and not its coherent counterpart (CVaR)? 
We answer this question by listing a few shortcomings of CVaR that UBSR overcomes \cite{giesecke2008measuring}. 
First, since CVaR is the conditional expectation beyond VaR, it ignores outcomes up to VaR, and assigns the same weight to all outcomes beyond VaR. Intuitively, if the underlying r.v. is associated with the loss of a financial position, it makes sense to have an increasing function of the losses incurred in measuring risk. In contrast, UBSR has an inbuilt function to weigh different outcomes and in the preceding financial example, one could choose this function underlying UBSR to be strictly increasing (e.g., exponential, square, etc). 
Next, from a technical viewpoint, CVaR does not satisfy the ``invariance under randomization'' property, which is made precise later. 

To summarize, the aim is to employ a risk measure that helps diversify, takes large losses into account and is invariant under randomization. Then, even though VaR is invariant under randomization, it is not preferable as it is not convex, while CVaR is convex but not invariant under randomization. In contrast, UBSR is convex and also invariant under randomization.

Risk measures have been extensively studied in the literature from the viewpoint of the attitude they capture when the underlying model information is available. On the other hand, in practical applications, such model information is seldom available. In comparison to risk attitude, optimization of risk with uncertainty has not received as much research attention in the literature. 
In the domain of supervised learning, a representative list of recent works is as follows: \cite{holland2021learning, holland2022spectral,fan2017learning,lee2020learning,mehta2023stochastic,khim2020uniform,frohlich2024risk}.
CVaR, spectral risk measures, and distortion risk measures have been considered in the aforementioned references. However, to the best of our knowledge, UBSR has not been explored in a supervised learning context. Our work aims to fill this gap.

A combination of UBSR and empirical risk minimization is challenging 
because UBSR is a non-linear function of the underlying distribution unlike standard expectation based risk measures. and a gradient oracle is not available for optimizing UBSR in the space of distributions. 

The main contribution of our work is framing the UBSR minimization problem as an optimization problem over distributions, and not random variables as is typically done. This framework enables plugging in standard optimization procedures to construct novel algorithms. In particular, we derive the gradient for the UBSR function w.r.t. its distribution argument, and update the model based on a black-box procedure that minimizes an expected risk. The structure of the UBSR loss function naturally suggests a bisection-like algorithm for making the model updates. We establish bounds on `excess UBSR' based on the error of the black-box procedure minimizing the expected risk -- which can be based on linear/kernel models or even deep networks. 

As a secondary contribution, we derive concentration bounds for UBSR estimation. Using the classic sample average approximation (SAA) estimator from independent and identically distributed (i.i.d.) samples we exhibit exponential tail bounds when the underlying distribution is either sub-Gaussian or sub-exponential.  These results are useful in establishing excess UBSR bounds mentioned earlier.

\paragraph{Related work.}
In the context of UBSR estimation, we mention the following works: \cite{dunkel2010stochastic,zhaolin2016convex,JMLR-LAP-SPB,hegde2021ubsr}. In the first two references, the authors propose stochastic approximation and sample average approximation (SAA) type schemes for UBSR estimation, and provide asymptotic guarantees for their estimators. \cite{JMLR-LAP-SPB} provides concentration bounds for SAA estimator of UBSR using a Wasserstein distance approach, while \cite{hegde2021ubsr} do the same for the stochastic approximation-based UBSR estimator. In comparison to these references, we mention the following aspects of our concentration bounds for the SAA estimator: (i) unlike the unified approach of \cite{JMLR-LAP-SPB} that appeals to several risk measures, we adopt a direct approach to arrive at our concentration result for UBSR. The advantage with such an approach that is tailor-made for UBSR is that the constants are improved; (ii) The bound in \cite{hegde2021ubsr} for a stochastic approximation scheme requires knowledge about the growth rate of the loss function, while we do not assume this information.

In the context of UBSR optimization, to the best of our knowledge, there is no prior work. Other risk measures such as CVaR, spectral risk measures, distortion risk measures and a risk measure based on cumulative prospect theory have been considered in \cite{holland2021learning, holland2022spectral,fan2017learning,lee2020learning,mehta2023stochastic,khim2020uniform,frohlich2024risk}.
In the aforementioned references, the approach is to consider a parameterized class of models, and use a stochastic gradient descent (SGD) type algorithm to find the best parameter. In contrast, we develop a non-parametric framework that finds the UBSR-optimal model by directly working with the distributions of the prediction losses. 
The work of \cite{harikrishna2024consistent} is closely related since they consider non-linear evaluation metrics for classification. Their setting involves confusion matrices, which are finite-dimensional, while we consider prediction loss distributions, which are infinite-dimensional.  

The rest of this paper is organized as follows:
Section \ref{sec:ubsr} provides an introduction to the UBSR risk measure. Section \ref{sec:est} describes UBSR estimation and presents concentration bounds for SAA-type UBSR estimator.
Section \ref{sec:opt} describes UBSR optimization in the context of a regression problem, the gradient oracle, a bisection-type algorithm and a bound on the error. Section \ref{sec:proofs} provides detailed proofs. Finally, Section \ref{sec:conclusions} provides the concluding remarks.

\section{Utility-based shortfall risk (UBSR)}
\label{sec:ubsr}

For a r.v. $Z$, the utility-based shortfall risk $SR_\alpha(Z)$ is defined as follows:
\begin{align}
	\jsr(Z) \triangleq \inf\left\{t \in \R \mid \E\left( \ell(-t+Z)\right)\le \lambda\right\},
	\label{eq:ubsr-def}
\end{align}
where $\ell:\R\rightarrow \R$ is a utility function (increasing and convex).

UBSR is a convex risk measure, when viewed as a function of r.v.s. Formally, a risk measure $\rho(\cdot)$ is convex if $\rho(\alpha X + (1-\alpha)Y) \le \alpha \rho(X) + (1-\alpha)Y$ for all r.v.s $X,Y$.
However, UBSR is not convex when viewed as a function of distributions and this claim is formalized in the lemma below.
\begin{lemma}
\label{lem:ubsr-notconvex}
    Let $\ell(z)=\max(0,z), z\in \R$. Let $F_1, F_2$ denote two uniform distributions in $[0,10]$ and $[10,20]$, respectively. Let $Z_1, Z_2$ be random variables with distributions $F_1, F_2$, respectively and let $\bar Z$ be a random variable with mixture distribution $\frac{F_1+F_2}{2}$. Then, for $\lambda=2$,
    \[ \jsr(\bar Z) > \frac{1}{2} \jsr(Z_1) + \frac{1}{2}\jsr(Z_2).\]
\end{lemma}
\begin{proof}
    See Section \ref{pf:ubsr-notconvex}.
\end{proof}
We shall establish later that UBSR is a pseudolinear function of distributions. In particular, UBSR is neither convex nor concave from a distributional viewpoint. This fact has a bearing on the UBSR optimization algorithms that we propose, see Section \ref{sec:opt} for the details.

We mentioned earlier that UBSR is invariant under randomization, while the popular CVaR risk measure is not. We shall formalize this statement now.
Let $\rho(X)$ denote a risk measure associated with a random variable $X$, e.g., $X$ is a financial position.
Let$\mathcal{X}$ be the set of positions, and let 
		$\mathcal{A}_\rho = \{ X \mid \rho(X)\le 0\}$ denote the set of acceptable positions, that is, those where the risk is within permissible limits.
Then $\rho(\cdot)$ is said to be invariant under randomization if for any $X_1,X_2 \in \mathcal{A}_\rho$ and any $\alpha \in [0,1]$, the financial position $X_\alpha \in \mathcal{A}_\rho$, where $
		X_\alpha =
			X_1 \text{with probability      } \alpha$ and $
			X_2$ otherwise.

For the sake of analysis, we shall make the following assumption.
\begin{assumption}\label{ass:bounds-for-sr}
	There exist $\tl, \tu \in \R$ and $\eta>0$ such that $\E\left( \ell(-\tl+Z)\right) \ge \lambda+\eta$, and $\E\left( \ell(-\tu+Z)\right) \le \lambda-\eta$. 	
\end{assumption}
A similar assumption has been made in previous work on UBSR estimation/optimization, cf. \citep{zhaolin2016convex,hegde2021ubsr}.
Under the above assumption, $\sr$ is the unique solution to $q(t) = \lambda$, which is made formal in the lemma below.
\begin{lemma}
    Under \Cref{ass:bounds-for-sr}, $\sr$ is the unique root of $q(t)\triangleq \E\left( \ell(-t+Z)\right) - \lambda$.
\end{lemma}
The reader is referred to 
    \citep{zhaolin2016convex} for a proof of the claim above.

\section{UBSR estimation}
\label{sec:est}
In \citep{zhaolin2016convex,JMLR-LAP-SPB}, the authors estimate $\sr$ using i.i.d. samples $\{Z_1, \ldots, Z_n\}$ of $Z$ as the solution to the following problem:
\begin{align}
	\min_{t\in \R} t \textrm{~~ subject to ~~} \frac{1}{n} \sum_{i=1}^{n} \ell\left( Z_i - t\right) \le \lambda.
    \label{eq:ubsrest-def}
\end{align}
We shall use $\srest$ to denote the solution of the above problem, 
where $\z$ is a $n$-vector with components $Z_i,\ i=1,\ldots,n$.
From \eqref{eq:ubsr-def}, it is apparent that $\srest$ is a SAA approximation of UBSR,


\subsection{Concentration bounds: sub-Gaussian case}
In addition to \Cref{ass:bounds-for-sr}, we make the following assumptions for the sake of analysis.
\begin{assumption}\label{ass:loss-Lipschitz}
	The utility function $\ell$ is $G$-Lipschitz, convex and strictly increasing.  	
\end{assumption}
\begin{assumption}\label{ass:subGauss}
	The r.v. $Z$ is sub-Gaussian\footnote{
A r.v. $Z$ with mean $\mu$ is sub-Gaussian if there exists a $\sigma >0$ such that
\[ \E\left(\expo{\lambda (X-\mu)}\right) \le \expo{\dfrac{\lambda^2 \sigma^2}{2}} \text{ for any } \lambda \in \R.\]
} with parameter $\sigma$.
\end{assumption}
Previous works on UBSR estimation have made an assumption similar to \Cref{ass:loss-Lipschitz}. In financial applications, it makes sense to have a strictly increasing loss function, while the convexity requirement is reasonable in light of the observation that UBSR is a convex risk measure. 
Finally, a tail assumption such as sub-Gaussianity in \Cref{ass:subGauss} is necessary to derive concentration bounds, and is common to previous works on risk estimation.

Define
\begin{align}
 q(t)\triangleq \E\left[ \ell(Z-t)\right] - \lambda, \textrm{~ and ~} q_n(t) \triangleq \frac{1}{n} \sum_{i=1}^{n} \ell\left( Z_i - t\right) - \lambda. \label{eq:qn-q}
\end{align}
We convert an upper bound on $|q_n - q|$ into a bound on $\left|\srest-\sr\right|$ by exploiting the convexity and monotonicity of $\ell$, see Figure \ref{fig:srest-illustration} for an illustration of the argument.

\begin{figure}
    \centering
    \scalebox{0.7}{
\begin{tikzpicture}
  \draw[red,thick] (0,1) -- (10,-3);
  \draw[red,thick] (0,3) -- (10,-1);
  
  \fill[red!10] (0,3) -- (10,-1) -- (10,-3) -- (0,1);
  \draw[red,ultra thick] (0,2) -- (10,-2);
  \draw[thick] (0,0) -- (10,0);
  \draw[thick] (0,2) -- (10,2);
  \draw[thick] (0,-2) -- (10,-2);
\draw[<->, thick] (9.5,2) -- (9.5,0) node[midway, right] {$\eta$};

  \node[anchor=north] at (0,0) {$\tl$};
  \node[anchor=north] at (10,0) {$\tu$};
  \node[anchor=north] at (5,0) {$\jsr$};

  \draw[<->, thick] (5,1) -- (5,0) node[midway, right] {$c$};


  \draw[<->, thick] (2.5,-3) -- (7.5,-3) node[midway, below] {$\frac{2c}{\eta}(\tu-\tl)$};
  \draw[dotted, thick] (2.5,0) -- (2.5,-3);
  \draw[dotted, thick] (7.5,0) -- (7.5,-3);
\end{tikzpicture}}
    \caption{A graphic illustration of the relation between $|q_n - q|$ and $\left|\srest-\sr\right|$. The dark red line corresponds to the function $q(t)$. If its estimate $q_n(t)$ lies in shaded red region, the difference between the solution of $q(t)=0$ and $q_n(t)=0$ can be bounded.}
    \label{fig:srest-illustration}
\end{figure}

For bounding $|q_n-q|$, we require the following result, which shows that a Lipschitz function of a sub-Gaussian r.v. is sub-Gaussian. The reader is referred to Lemma B.1 of \citep{ghosh2024concentration} for a proof.
\begin{lemma}
\label{lem:lipschitz-subg} 
Suppose $Z$ is a sub-Gaussian r.v. with parameter $\sigma$ and $f$ is a $G$-Lipschitz function. Then, $f(Z)$ is sub-Gaussian with parameter $2G\sigma$. 
\end{lemma} 

Since $\ell$ is $G$-Lipschitz and $Z$ is sub-Gaussian, from the lemma above, we have following bound with probability $(1-\delta)$, for any $\delta\in (0,1)$ and for any $t\in [\tl,\tu]$:
\begin{align}
    |q_n(t) - q(t)| \le 2G\sigma \sqrt{\frac{\log 1/\delta}{n}},
    \label{eq:hn-conf}
\end{align}
where $q(t)= \E\left[ \ell(Z-t)\right] - \lambda$, and $q_n(t)$ is defined in \eqref{eq:ubsrest-def}.

\begin{theorem}
\label{thm:ubsr-est-subGauss}
Suppose \Crefrange{ass:bounds-for-sr}{ass:loss-Lipschitz} hold. 
For any $\delta \in (0,1)$, the UBSR estimator $h_n$ satisfies     
   \[ \left|\srest-\sr\right|\le \frac{2G\sigma (\tu-\tl)}{\eta} \sqrt{\frac{\log 1/\delta}{n}} \textrm{ w.p. } (1-2\delta),\] 
   where $\tl,\tu,\eta$ are specified in \Cref{ass:bounds-for-sr}, $G, \sigma$ are given in \Crefrange{ass:loss-Lipschitz}{ass:subGauss}, respectively.
\end{theorem}
\begin{proof}
    See Section \ref{sec:ubsr-est-proof}.
\end{proof}

\subsection{Concentration bounds: sub-exponential case}
In this section, we provide a variant of the result in \Cref{thm:ubsr-est-subGauss} that caters to sub-exponential distributions.
As in the sub-Gaussian case, we first establish that Lipschitz functions of sub-exponential remains sub-exponential.
\begin{lemma}
Suppose $Z$ is a sub-exponential r.v. with parameter $K$, i.e., $\E\left[\exp\left(|X|/K\right)\right] \leq  2 $ and $f$ is a $G$-Lipschitz function. Then, $f(Z)$ is sub-exponential with parameter $4eGK$. 
\end{lemma}
\begin{proof}
    See Section \ref{sec:ubsr-est-subexp-proof}.
\end{proof}
Since $\ell$ is $G$-Lipschitz and $Z$ is sub-exponential, from the lemma above, we have following bound with probability $(1-\delta)$, for any $\delta\in (0,1)$ and for any $t\in [\tl,\tu]$:
\begin{align}
    |q_n(t) - q(t)| \le 4eGK \frac{\log 1/\delta}{n},
    \label{eq:hn-conf}
\end{align}
where $q(t)= \E\left[ \ell(Z-t)\right] - \lambda$, and $q_n(t)$ is defined in \eqref{eq:ubsrest-def}.
\begin{theorem}
\label{thm:ubsr-est-subExp}
Suppose \Crefrange{ass:bounds-for-sr}{ass:loss-Lipschitz} hold. 
For any $\delta \in (0,1)$, the UBSR estimator $h_n$ satisfies     
   \[ \left|\srest-\sr\right|\le \frac{4e G K (\tu-\tl)}{\eta} \frac{\log 1/\delta}{n} \textrm{ w.p. } (1-2\delta),\] 
   where $\tl,\tu,\eta$ are specified in \Cref{ass:bounds-for-sr}, $G, \sigma$ are given in \Crefrange{ass:loss-Lipschitz}{ass:subGauss}, respectively.
\end{theorem}
\begin{proof}
    See Section \ref{sec:ubsr-est-subexp-proof}.
\end{proof}
\begin{remark}[Comparison to existing bounds]
    In \cite{hegde2021ubsr}, the authors employ a stochastic approximation-based scheme to estimate UBSR, and derive a bound that is of similar order as that in \Cref{thm:ubsr-est-subGauss}. However, to achieve this rate, they require knowledge of the rate at which the loss function grows.
    For the case wheree a lower bound onthe growth rate of the loss functions function is not known, the authors in \cite{hegde2021ubsr} provide a weaker $O\left(\frac{1}{n^\alpha}\right)$ bound, where $\alpha<1$. Our $O\left(\frac{1}{\sqrt{n}}\right)$ bound in \Cref{thm:ubsr-est-subGauss}   does not require the aforementioned information about the loss function.

    Next, in \cite{JMLR-LAP-SPB}, the authors use a Wasserstein distance approach to derive a concentration bound for the UBSR estimator defined in \eqref{eq:ubsrest-def}. To elaborate, they relate the UBSR estimation error $\left|\srest-\sr\right|$ to the Wasserstein distance between the empirical and the true distribution. A bound on the latter implies a bound on the UBSR estimation error. The resulting constants in the bound derived using such an approach are conservative, as the entire distribution is approximated with Wasserstein distance. On the other hand, our approach is more direct, and features improved constants in comparison to \cite{JMLR-LAP-SPB}.
\end{remark}







\section{UBSR optimization}
\label{sec:opt}

Consider a regression problem, where one aims to find a model $\hat y:\X\rightarrow\R$ that minimizes $Z$, a measure of the deviation of the prediction $\hat y(X)$ from the true target $Y$. Standard machine learning algorithms are designed to minimize the expected value of $Z$, but are not suited for other ways of aggregation. In this paper, we consider the objective of minimizing the shortfall risk of the deviation, i.e., to solve the following problem: 
$\min_{\hat{y}} \jsr(Z) $, where 
$Z=(Y-\hat y(X))^2$ is the square loss\footnote{We fix $Z=(Y-\hat y(X))^2$ , but our results easily extend to the case where $Z$ is any non-negative function of $Y$ and $\hat y(X)$ is convex in $\hat y(X)$.}.

We frame this problem as one of minimizing a function of the joint distribution of the model prediction $\hat y(X)$ and the true target $Y$, over the parameters of the model $\hat y$. We allow $\hat y$ to also be randomized, i.e., the model output $\hat y(x)$ is a real valued random variable. Our goal is to find a model $\hat y$ that minimizes  $\Gamma[F_Z]$, where $F_Z$ is the CDF of $Z$. If $\Gamma$ is a linear function of its argument distribution $F_Z$, the resulting objective corresponds to a variant of the expectation of $Z$, and ERM-based standard optimization algorithms can be applied in a straightforward fashion. However, if $\Gamma$ is a non-linear function of its argument, standard ERM algorithms cannot be applied, as the true risk of a given model is not approximated by a sample average. 

We extend the notion of UBSR to a function $F$, which may not necessarily be a CDF of some r.v.,  as follows
\[
\overline{\jsr}[F] = \inf \{t \in \R: \int_{\R} \ell(z-t) dF(z) \leq \lambda \},  
\]
where $\ell:\R\rightarrow\R$ is a convex monotonically increasing function and $\lambda \in \R$ is some fixed constant. For the case where $F$ is the CDF of a r.v. $Z$, we have $\overline{\jsr}[F] = \jsr[Z]$. For any function $F:\R\rightarrow\R$, we define $L_F:\R\rightarrow\R$, as follows: 
\begin{equation}
L_F(t) \triangleq\int_{\R} \ell(z-t) dF(z).
    \label{eq:Lf}
\end{equation}
Notice that, when $F$ is a valid CDF, $L_F(t)$ is a monotone decreasing function of $t$, and the derivative of $L_F$ is negative as $\ell$ is a strictly increasing function.

Consider the set of distributions (CDFs), related to the distribution of the data tuple $(X,Y)$:
\[ \mathcal{F} = \{F_Z: Z = (Y-\widehat Y(X))^2,  \text{ where } \hat Y(X) \sim 
 h(X) \text{ for some } h \in \mathcal{H} \}.\]
In the above, $\mathcal{H}$ is some set of randomized models, i.e., set of functions that take an input $\x\in \mathcal{X}$ and output a distribution over $\R$.  Each model $h \in \H$ corresponds to an element (not necessarily unique) in the set $\F$. We consider a family of randomized models $\H$ that is closed under mixing, i.e., if $h_1, h_2 \in \H$, then its mixture $h_3 = \alpha h_1 + (1-\alpha)h_2 \in \H$ for any $\alpha \in [0,1]$. This property ensures that the set of distributions $\F$ is convex. The set $\F$ is an underlying property of the distribution of the data $(X,Y)$, and can be used to reformulate our original problem as follows:
\[
\inf_{h \in \H} \jsr[(Y-\widehat Y(X))^2] = \inf_{F \in \mathcal{F}} \overline\jsr[F],
\]
where for any $\x \in \X, \widehat Y(\x) \sim h(\x)$. The problem on the RHS of the above expression seems technically easier, as it corresponds to optimizing only over distributions on the real line, as opposed to randomized real-valued models over the input space. An important property of $\overline\jsr$ that makes this problem tractable is that it is a pseudo-linear function \citep{mangasarian1975pseudo}  of its argument $F$.
We define pseudo-linearity next.
\begin{defn}
A function $J:\mathcal{C} \rightarrow \R$ is pseudo-convex if it is increasing in any direction  where it has positive directional derivative, i.e., for any given $\x \in \mathcal{C} \subseteq \R^d$ and $\mathbf{u} \in \R^d$,
\[
\lim_{\epsilon \rightarrow 0}\frac{J(\x+\epsilon \mathbf{u}) - J(\x)}{\epsilon} > 0 \implies J(\x+\alpha \mathbf{u}) \text{ is an increasing function of }\alpha
\]
\end{defn}
A function $J$ is pseudo-concave if $-J$ is pseudo-convex. A function is pseudo-linear if it is both pseudo-convex and pseudo-concave.
Figure \ref{fig:pseudolinear} provides an  illustration of pseudo-linearity, using the level sets. 
Intuitively, pseudo-linearity is equivalent to the function being monotonic along any line.

The result below establishes pseudo-linearity of UBSR.
\begin{theorem}
\label{thm:pseudo-linear}
    Let $F, F'$ be valid distributions over $\R$. Then $\overline\jsr[\alpha F + (1-\alpha)F']$ is a monotone function of $\alpha \in [0,1]$.
\end{theorem}
\begin{proof}
    See Section \ref{pf:pseudo-linear}.
\end{proof}

\begin{figure}
    \centering
    \scalebox{0.7}{
\begin{tikzpicture}
  \def\a{4}
  \def\b{2}

  \begin{scope}[rotate=150]
    \draw[thick, red] (0,0) ellipse [x radius=\a, y radius=\b];

    \foreach \i in {0,...,29} {
      \pgfmathsetmacro{\angle}{\i * 360 / 30} 
      \pgfmathsetmacro{\x}{\a * cos(\angle)}
      \pgfmathsetmacro{\y}{\b * sin(\angle)}
      \coordinate (P\i) at (\x,\y);
    }

    \foreach \i/\j in {0/2, 3/28, 4/26, 5/24, 6/22, 7/20, 8/17, 9/15, 10/13} {
      \draw[blue, thick] (P\i) -- (P\j);
    }
  \end{scope}
\end{tikzpicture}}
    \caption{A graphic illustration of the contours of a pseudolinear function $J$ with domain inside the red ellipse. The blue lines correspond to contours, $J=k$ for some values $k$. Notice that the contours are linear, non-parallel and non-intersecting inside the domain.}
    \label{fig:pseudolinear}
\end{figure}

To optimize $\overline\jsr[F]$ over ${F \in \mathcal{F}}$, and use standard iterative optimisation methods, we need the gradient of $\overline\jsr[F]$ w.r.t. $F$. We would also need oracle access to the set $\mathcal{F}$, which is quite complex and cannot be expressed in a closed form.
We design a linear minimization oracle, which is capable of minimizing linear functions of the distribution $F$ over the set $\mathcal{F}$. This corresponds to minimizing an expected risk and can be done using standard empiricial risk minimization (ERM) techniques. With access to a linear minimisation oracle, we can minimise pseudo-linear functions like $\overline\jsr$ using a bisection-like algorithm, cf. \cite{harikrishna2024consistent}.

We present the gradient of $\overline\jsr$ and the linear minimization oracle for $\mathcal{F}$ in sections \ref{sec:gradient} and \ref{sec:LMO}, respectively. This oracle is then used in a bisection-like algorithm in Section \ref{sec:bisection}.

\subsection{Gradient of UBSR}
\label{sec:gradient}

The objective $\overline\jsr$ is to be minimized over a set of distributions, and hence its gradient is technically an infinite dimensional object that is best expressed via an inner product with a function $u:\R \rightarrow \R$. The result below presents an expression for the so-called Gateaux derivative, which generalizes the directional derivative, of UBSR.

\begin{lemma}
\label{lem:gradient_oracle}
Let $F$ be a  CDF over $\R$ with finite $\overline\jsr[F]$, and $u:\R\rightarrow \R$.
Then,
\begin{equation}
    \label{eqn:gradient_oracle}
    \langle \nabla \overline\jsr[F] , u \rangle = \lim_{\epsilon \rightarrow 0} \frac{1}{\epsilon} \left(\overline\jsr[F+\epsilon u]-\overline\jsr[F] \right) =  c L_u(\overline\jsr[F])
\end{equation}    
for some scalar $c>0$. The function $L_u$ is as defined in \eqref{eq:Lf}.
\end{lemma}
\begin{proof}
See Section \ref{pf:ubsr-gradient-expr}.
\end{proof}

\subsection{Linear Minimisation Oracle}
\label{sec:LMO}

Implementing standard iterative constrained optimization algorithms (e.g. minimize $\overline\jsr[F]$ such that $F \in \mathcal{F}$) that use linear minimization oracles requires solving the following problem at the current iterate $F = F_Z$: 
\begin{align*}
\argmin_{u \in \mathcal{F}} \langle \nabla \overline\jsr[F] , u \rangle 
&= 
\argmin_{u \in \mathcal{F}}  L_u(\overline\jsr[F])  = 
\argmin_{u \in \mathcal{F}}  \int \ell(z-\overline\jsr[F])du(z)  \\
&\equiv
\argmin_{h \in \mathcal{H}} \mathbb \E_{(X,Y)\sim D} \E_{\widehat Y \sim h(X)}[\ell((Y-\widehat Y)^2 - \overline \jsr[F])].\numberthis\label{eq:ubsr-opt-grad}
\end{align*}
The first equality above is from Lemma \ref{lem:gradient_oracle}. The equivalence in the last line is due each point in $u \in \F$ corresponding to (at least) a model $h \in \H$. The expectation in the last line above is over the randomness of $X,Y$ and the randomized model $\widehat Y$. 
We can approximate the expectation in \eqref{eq:ubsr-opt-grad} by replacing it with a sample average evaluated over a subset of the training data $S = \{(\x_1, y_1), \ldots, (\x_n,y_n) \}$. We can also approximate $\overline \jsr[F] = \jsr(Z)$ by an estimate from another disjoint subset of the training data $S'=\{(\x_{n+1}, y_{n+1}), \ldots, (\x_{2n},y_{2n}) \}$. 
Thus, we obtain
\[
\argmin_{u \in \mathcal{F}} \langle \nabla \overline\jsr[F] , u \rangle 
\equiv
\argmin_{h \in \mathcal{H}} \sum_{i=1}^n \mathbb E_{\hat y_i \sim h(x_i)} \left[ \ell( (y_i - \hat y_i)^2 - \gamma) \right],
\]
where $\gamma$ is the estimate of the UBSR of the current distribution $F_Z$ calculated from $S'$. While $\hat y$ is a random variable, it can be easily shown that one of the solutions to the above problem will always be a deterministic model $\hat y:\X \rightarrow \R$, and the above problem can be simplified further. 

Assuming a linear parameterization of $\hat y(\x) = \w^\top \x$, the problem becomes:
\[
\min_{\hat y :\mathcal{X} \rightarrow \mathbb{R}} \sum_{i=1}^n \mathbb   \ell( (y_i - \hat y(x_i))^2 - \gamma) = 
\min_{\mathbf{w} \in \R^d} \sum_{i=1}^n \mathbb   \ell( (y_i - \w^\top \x_i )^2 - \gamma)
\]

The above objective is a convex function of $\w$, as $\ell$ is convex, monotone increasing and the function $(y_i - \w^\top \x_i)^2$ is convex in $\w$.  Thus, the desired linear minimization oracle can be provably implemented using classic optimization procedures such as gradient descent.


\subsection{Bisection for UBSR Minimization}
\label{sec:bisection}
We now have a handle on both the objective function $\overline\jsr$ and the constraint set $\mathcal{F}$, and potentially enable constrained optimization algorithms that use a linear minimization oracle for the constraint set and a gradient oracle for the objective. A popular algorithm for this problem is the Frank-Wolfe (or conditional gradient) algorithm \cite{jaggi2013revisiting}, which is designed for convex and smooth objectives. However, the UBSR objective $\overline\jsr[F]$ is not a convex function of the distribution $F$ (see Lemma \ref{lem:ubsr-notconvex}). Instead, UBSR is a pseudo-linear function, and such functions are amenable to be minimized by bisection-type algorithms.
Algorithm \ref{alg:bisection_UBSR} presents the pseudocode for a bisection algorithm for UBSR optimization.

Algorithm \ref{alg:bisection_UBSR} exploits the monotonicity of $\overline\jsr$ over straight lines using a binary-search-like procedure. In every iteration of this algorithm, $\alpha_t$ and $\beta_t$ provide the lower and upper bound on the minimum value of UBSR. The value $\gamma_t$ is the current guess of the minimum UBSR. In each iteration, the search space of the minimal UBSR is narrowed to either $[\alpha_t, \gamma_t]$ or $[\gamma_t, \beta_t]$.  The key decision variable is the estimated UBSR $\gamma$ of the model $g$ that minimizes $\sum_{i=1}^n \mathbb \ell( (y_i - g(x_i))^2 - \gamma_t)$. The property of pseudo-linearity of $\overline\jsr$ ensures that if $\gamma_t > \gamma$, then the minimal UBSR is larger than $\gamma_t$. 

An important consequence of the pseudo-linearity of $\overline\jsr$ is that its minimizer over $\F$ is guaranteed to lie on the boundary of $\F$. Elements $F$ in the interior of $\F$ correspond to randomized models, while those on the boundary correspond to deterministic models. This is the reason why the final model returned by Algorithm \ref{alg:bisection_UBSR} is deterministic, despite the class of models $\H$ including randomized models.

\begin{figure}
\begin{algorithm}[H]
\caption{Bisection for UBSR Minimization}\label{alg:bisection_UBSR}
\begin{algorithmic}[1]
\STATE \textbf{Input:} Training set $\{(\x_1,y_1),\ldots,(\x_{2n},y_{2n})\}$. Monotone increasing convex function $\ell$. $\lambda \in \R$.
Hypothesis class of models $\H$.
\STATE Initialise $h_0$ s.t. $h_0(\x)=0$ for all $\x \in \X$. Let $\alpha_0=0$.
\STATE Compute squared loss of $h_0$ on data : $\z^0_i = (h_0(\x_{n+i}) - y_{n+i})^2$, for $i=1,\ldots,n$.
\STATE Estimate the UBSR of $(Y-h_0(X))^2$ from samples: $\beta_0 = \srarg{n}(\z^0)$

\STATE \textbf{For} $t = 1$ \textbf{to} $T$:
\STATE ~~~~~~~~ $\gamma_t = (\alpha_{t-1} + \beta_{t-1})/2$
\STATE ~~~~~~~~ Compute new model $g_t=  \argmin_{g \in \H} \sum_{i=1}^n \mathbb   \ell( (y_i - g(x_i))^2 - \gamma_t)$
\STATE ~~~~~~~~ Compute squared loss of $g_t$ on data :
$\z^t_i = (g_t(\x_{n+i}) - y_{n+i})^2$, for $i=1,\ldots,n$.
\STATE ~~~~~~~~ Estimate the UBSR of $(Y-g_t(X))^2$ from samples: $\gamma = \srarg{n}(\z^t)$
\STATE ~~~~~~~~ \textbf{If} $\gamma < \gamma_t$:
\STATE ~~~~~~~~~~~~~~~~ $h_{t} = g_t$;  $\alpha_t = \alpha_{t-1}$ ; $\beta_t = \gamma_t$.
\STATE ~~~~~~~~ \textbf{Else}:
\STATE ~~~~~~~~~~~~~~~~ $h_{t} = g_{t}$;  $\alpha_t = \gamma_{t}$ ; $\beta_t = \beta_{t-1}$.
\STATE ~~~~~~~~ \textbf{End If}
\STATE \textbf{End For}
\STATE Return $h_T$.
\end{algorithmic}
\end{algorithm}
\vspace{-10pt}
\end{figure}

\subsection{Algorithm Analysis}

The algorithm described in the previous section is intuitively appealing and easy to implement. However, a theoretical analysis to arrive at an excess UBSR bound has to handle several errors that could arise due to finite sample sizes. In particular, these errors are in the computation of the model $g_T$ as well as in the estimation of its UBSR. In this section, we show that the afrorementioned errors in each iteration due to finite sample sizes are not catastrophic, i.e., the UBSR sub-optimality of the final model $h_T$ can be bounded by a constant (that is independent of $T$) times the sub-optimality of the optimization step for computing $g$ and its UBSR estimation error. The proof of this takes cues from \cite{harikrishna2024consistent}, who show a similar result for a ratio-of-linear evaluation metrics (such as the F-measure) used in classification problems. We extend the argument to regression problems with an infinite-dimensional constraint set $\F$. In contrast, the authors in \citep{harikrishna2024consistent} consider classification problems with a finite-dimensional set of confusion matrices. 

\begin{theorem}
\label{thm:main-theorem}
Suppose \Cref{ass:loss-Lipschitz} holds.
Let the optimization step (in line 7) and estimation step (in line 9) of Algorithm \ref{alg:bisection_UBSR}  be such that the following hold with probability $1-\delta$ for some $\rho,\rho ' \in \R$. For all $t\leq T$,  
\begin{align*}
    \E\left[\ell((Y-g_t(X))^2-\gamma_t) \right] - \min_{g\in \H} \E\left[\ell((Y-g(X))^2-\gamma_t) \right]
    &\leq  \rho, \\
    |\srarg{n}(\z^t) - \jsr((Y-g_t(X))^2) | 
    &\leq \rho'.
\end{align*}
Then, the final model $h_T$ returned by the algorithm satisfies  the following bound w.p. $(1-\delta)$:
\begin{align}
\jsr((Y-h_T(X))^2) \leq \min_{h \in \H} \jsr((Y-h(X))^2) + C_1 \rho + C_2 \rho' + C_3 2^{-T},
\label{eq:main-opt-bd}
\end{align}
for some absolute constants $C_1, C_2, C_3$.
\end{theorem}
\begin{proof}
    See Section \ref{pf:main-theorem}.
\end{proof}
The term $\rho$ can be bounded using classic ERM risk bounds, e.g. \cite{bartlett2002rademacher}. The term $\rho'$ can be bounded using concentration inequalities for $\srarg{n}$ derived in Section \ref{sec:est}.
\begin{corollary}
\label{cor:linear-models-bound}
Suppose $\|X\|\leq B_1$, and $|Y|\leq B_2$. Let $\H=\{h:\R^d \rightarrow \R: h(\x) = \w\tr \x \text{ for some } \w\in \R^d, \|\w\| \leq B_3\}$ be the set of linear functions with bounded norm. Suppose \Cref{ass:loss-Lipschitz} holds. Then, the bounds in Theorem \ref{thm:main-theorem} hold with 
\[ \rho' = C_4 \sqrt{\frac{\log 1/\delta}{n}}, \,
    \rho=\frac{C_5}{\sqrt{n}}+ C_6 \sqrt{\frac{\log(1/\delta)}{n}},\]
for some absolute constants $C_4, C_5, C_6$.    
\end{corollary}
\begin{proof}
    See Section \ref{pf:cor-lin-models}.
\end{proof}
The power of Theorem \ref{thm:main-theorem} and Algorithm \ref{alg:bisection_UBSR} lies in its ability to convert any black box machine learning procedure (say deep neural networks or kernel methods) into an algorithm for minimizing the UBSR, where we can transfer the guarantees for the black box procedure into guarantees for excess UBSR.
In particular, variants of \Cref{cor:linear-models-bound} for the case of kernel methods and neural networks can be derived by combining \Cref{thm:main-theorem} with results from Section 9.3 of \cite{zhang2023mathematical} and Section 19.4 of \cite{anthony2009neural}, respectively. It is important to note that the bounds in \Cref{thm:main-theorem} are useful only if the errors $\rho, \rho'$ have upper bounds that vanish with the number of samples $n$.

\section{Proofs}
\label{sec:proofs}
\subsection{Proof of Lemma \ref{lem:ubsr-notconvex}}
\label{pf:ubsr-notconvex}
\begin{proof}
    Let $\E_i$ denote the expectation w.r.t. distribution $F_i$, for $i=1,2$, respectively.
    Notice that
    \begin{align*}
        \E_1(\ell(Z-t)) &= \int_\R \max(x-t,0) \indic{x\in [0,10]} \frac{1}{10} dx\\
        &= \frac{1}{10} \int_{10}^{20} \max(x-t,0)dx\\
        &= \begin{cases}
            \frac{1}{10} \int_{10}^{20} (x-t) dx & \textrm{ if } t<0\\[0.75ex]
            \frac{1}{10} \int_{t}^{20} (x-t)dx & \textrm{ if } t\in[10,20]\\[0.75ex]
            0 & \textrm{ otherwise}
        \end{cases}\\
        &= \begin{cases}
            (5-t)  & \textrm{ if } t<0\\[0.75ex]
            \frac{(10-t)^2}{20}  & \textrm{ if } t\in[10,20]\\[0.75ex]
            0 & \textrm{ otherwise.}
        \end{cases}
    \end{align*}
Along similar lines, 
    \begin{align*}
        \E_2(\ell(Z-t)) 
        &= \begin{cases}
            (15-t)  & \textrm{ if } t<0\\[0.75ex]
            \frac{(20-t)^2}{20}  & \textrm{ if } t\in[10,20]\\[0.75ex]
            0 & \textrm{ otherwise.}
        \end{cases}
    \end{align*}
From the foregoing, for $\lambda=2$, we obtain $\jsr(Z_1)$ and $\jsr(Z_2)$ as solutions to 
$\frac{(10-t)^2}{20}=2$ and $\frac{(20-t)^2}{20}=2$, respectively. Thus,
\begin{align}
    \jsr(Z_1)= 10-\sqrt{40}, \textrm{ and } \jsr(Z_2)= 20-\sqrt{40}.\label{eq:srz12}
\end{align}
Recall $\bar Z$ is a random variable with mixture distribution $\frac{F_1+F_2}{2}$.
It is easy to see that $\bar Z$ is uniform in $[0,20]$. A simple calculation shows that $\jsr(\bar Z)$ is the solution of
\[\frac{(20-t)^2}{40}=2 \textrm{implying } \jsr(\bar Z)= 20-\sqrt{80}.\]
From the above and \eqref{eq:srz12}, we have 
\[\jsr(\bar Z) > \frac{1}{2} \jsr(Z_1) + \frac{1}{2}\jsr(Z_2) \textrm{ for } \lambda=2.\]
The claim follows.
\end{proof}

\subsection{Proof of Theorem \ref{thm:ubsr-est-subGauss}}
\label{sec:ubsr-est-proof}
\begin{proof}
    We first establish a bound on the derivative of $q$ at $\sr$. For notational simplicity, we shall use $t^*$ to denote $\sr$. Using \Cref{ass:bounds-for-sr} and convexity of $q$, we have
    \begin{align*}
        &-\eta \ge q(\tu) \ge q(t^*) + q'(t^*)(\tu-t^*)\\
        \Rightarrow & -\eta \ge q'(t^*)(\tu - t^*) \\
        \Rightarrow&\,\,q'(t^*) \le -\frac{\eta}{(\tu - t^*)}\le -\frac{\eta}{(\tu - \tl)},\numberthis\label{eq:hprime-bound}
    \end{align*}
    where the last inequality uses the fact that $\tl \le t^*$.

Let $c = 2L\sigma \sqrt{\frac{\log 1/\delta}{n}}$ denote the confidence width, see \eqref{eq:hn-conf}. 
Let 
\begin{align}
t_1=t^* - \frac{c(\tu-\tl)}{\eta} \textrm{ and } t_2=t^* + \frac{c(\tu-\tl)}{\eta}.    \label{eq:t1t2}
\end{align}
Then, we claim that 
\begin{align}
    q(t_1)\ge c \textrm{ and } q(t_2)\le c.\label{eq:ht1t2bound}
\end{align}
We now prove the claim stated above.
Using convexity of $h$,
\begin{align*}
    q(t_1) &\ge q(t^*) + q'(t^*)(t_1-t^*)\\
    &= - q'(t^*) \frac{c(\tu-\tl)}{\eta} \ge c,
\end{align*}
where the last inequality used \eqref{eq:hprime-bound}.

Next, using convexity of $q$ again, we have
\begin{align*}
q(t_2) &\le \frac{(t_2-t^*)}{(\tu-t^*)} q(\tu) + \frac{(\tu-t_2)}{(\tu-t^*)} q(t^*)\\
&\le  \frac{(t_2-t^*)}{(\tu-t^*)} (-\eta) \tag{using $q(t^*)=0$ and \eqref{eq:hprime-bound}}\\
& = \frac{-c(\tu-\tl)}{(\tu-t^*)} \tag{using definition of $t_2$}\\
& \le -c \tag{ $\tl \le t^* \le \tu$}.
\end{align*}
Hence, \eqref{eq:ht1t2bound} holds.

From \eqref{eq:hn-conf} and \eqref{eq:ht1t2bound} we have with probability $1-2\delta$,
\begin{align*}
    q_n(t_1) > 0 \text{ and } q_n(t_2) < 0.
\end{align*}
Hence $\srest \in [t_1, t_2]$ and thus
\begin{align*}
|\srest - \sr |  = |\srest - t^* | 
&\leq \max(|t_1 -  t^*|, |t_2 - t^*|) \\ 
& \leq \frac{c(\tu-\tl)}{\eta} \tag{using definitions of $t_1,t_2$ in \eqref{eq:t1t2}}
\end{align*}
The claim follows by noting $c = 2L\sigma \sqrt{\frac{\log 1/\delta}{n}}$ and $t^*=\sr$.
\end{proof}

\subsection{Proof of Theorem \ref{thm:ubsr-est-subExp}}
\label{sec:ubsr-est-subexp-proof}
\begin{lemma}
\label{lem:subexp-Lip}
Suppose $X$ is a sub-exponential r.v. with parameter $K$, i.e., $\E\left[\exp\left(|X|/K\right)\right] \leq  2 $ and $f$ is a $L$-Lipschitz function. Then, $f(X)$ is sub-exponential with parameter $4eLK$. 
\end{lemma}
\begin{proof}
    \begin{align*}
\E\left[ e^{\lambda|f(X) - \E[f(X)]|} \right] &= \E\left[ e^{\lambda|f(X) - \E[f(X')]|} \right] \quad \text{($X'$ is an independent copy of $X$)} \\
&\leq \E\left[ e^{\lambda |f(X) - f(X')|} \right] \quad \text{(Jensen's inequality)} \\
&\leq \E\left[ e^{\lambda L |X - X'|} \right] \quad \text{(Lipschitzness of $f$)} \\
&\leq \E\left[ e^{\lambda L (|X| + |X'|)} \right]
\\
&\leq \left(\E\left[ e^{\lambda L |X|} \right]\right)^{2} = \E\left[ e^{2\lambda L |X|} \right] \\
&\leq e^{4 e^2 \lambda^2 K^2 L^2} \textrm{ for } |\lambda| \le \frac1{4e L K} \quad \text{(Proposition 2.71. of \cite{vershynin2018high})}
\numberthis\label{eq:fX&X}
\end{align*}
Setting $\lambda = \frac{1}{4 e L K }$, we obtain 
\begin{equation}
\E\left[e^{|f(X)-\E[f(X)]|/K'}\right] \leq  2 \quad \text{with } K' = 4eLK.
\label{eq:fXleq2}
\end{equation}
In particular, using the above inequality, we have
\begin{align}
    \prob{|f(X)-\E f(X)|>\epsilon} \le 2 \exp\left(-\frac{\epsilon}{K'}\right).
\end{align}
\end{proof}

\paragraph{Proof of \Cref{thm:ubsr-est-subExp}}
\begin{proof}
    Follows by a completely parallel argument to the proof of \Cref{thm:ubsr-est-subGauss}, with \Cref{lem:subexp-Lip} in place of \Cref{lem:lipschitz-subg}.
\end{proof}

\subsection{Proof of Theorem \ref{thm:pseudo-linear}}
\label{pf:pseudo-linear}
\begin{proof}
The proof simply follows after observing that the domain of $\overline\jsr$ is the set of distributions, which is a convex set, and the level sets $\{F: \overline\jsr[F]=k\}$ are convex for any $k\in \R$. 

A key property that will be used in proving the latter fact is that $L_F(t)$ is a strictly decreasing function of $t$ and is $G$ Lipschitz when $F$ is a valid CDF.

To see the convexity of level sets of $\overline\jsr$, consider two distributions $F_1$ and $F_2$ such that $\overline\jsr[F_1]=\overline\jsr[F_2]=k$ for some $k\in \R$. Now, consider the distribution $F_\alpha= \alpha F_1 + (1-\alpha)F_2$ for some $\alpha \in [0,1]$. Then by linearity of $L_F$ w.r.t. $F$, we have
\[
L_{F_\alpha}(t) = \alpha L_{F_1}(t) + (1-\alpha) L_{F_2}(t).
\]
We also have by construction of $F_1$ and $F_2$,
\[
\inf (\{t:L_{F_1}(t) \leq \lambda \}) = \inf (\{t:L_{F_2}(t) \leq \lambda \}) = k.
\]
In other words, graphs of the strictly decreasing functions $L_{F_1}(t)-\lambda$ and $L_{F_2}(t)-\lambda$ both hit the x-axis exactly at $t=k$.  $L_{F_\alpha}$, a convex combination of $L_{F_1}$ and $L_{F_2}$, would also hit the x-axis exactly at $t=k$. Hence, $\overline \jsr[F_\alpha]=k$. As this holds for any $\alpha \in [0,1]$, the level sets of $\overline\jsr$ are convex.

Now we argue that convexity of level sets implies monotonicity along any straight line. If $\overline \jsr$ is not monotone along the line joining  $F,F'$, there exists $0\leq \alpha_1 < \alpha_2 < \alpha_3 \leq 1$  such that 
\[
\overline\jsr(\alpha_1 F + (1-\alpha_1) F') = \overline\jsr(\alpha_3 F + (1-\alpha_3) F') 
\neq \overline\jsr(\alpha_2 F + (1-\alpha_2) F')
\]
    
The above would mean that $\alpha_2 F + (1-\alpha_2)F'$ is not in the level set of $\overline \jsr$ that contains $\alpha_1 F + (1-\alpha_1)F'$ and $\alpha_3 F + (1-\alpha_3)F'$, which violates convexity of level sets of $\overline \jsr$ as $\alpha_1 < \alpha_2 < \alpha_3$.
    
\end{proof}

\subsection{Proof of Lemma \ref{lem:gradient_oracle}}
\label{pf:ubsr-gradient-expr}
\begin{proof}
We assume $L_u(\overline\jsr[F])$ is finite, as the case of $L_u(\overline\jsr[F])=\infty$ is trivial. \\
Notice that
    \begin{align*}
\overline\jsr[F+\epsilon u ]   
&= \inf (\{t: L_{F+\epsilon u}(t)   \leq \lambda \}) = \inf (\{t: L_{F+\epsilon u}(t)   = \lambda \})\\
&= \inf (\{t: L_F(t) + \epsilon \int_{\R} \ell(z-t) du(z)  = \lambda \}) \\
&= \inf (\{t: L_F(t) + \epsilon L_u(t) = \lambda \}) \\
&= \overline\jsr[F] + \inf (\{v: L_F(\overline\jsr[F]+v) + \epsilon L_u(\overline\jsr[F]+v) = \lambda \}). 
\end{align*}
By the mean value theorem, for every $v$, there must exist some $w_1(v)$ and $w_2(v)$ between $\overline\jsr[F]$ and  $\overline\jsr[F]+v$ such that 
\begin{align*}
    L_F(\overline\jsr[F]+v) &= L_F(\overline\jsr[F]) + L_F'(w_1(v))v, \textrm{ and } \\ L_u(\overline\jsr[F]+v) &= L_u(\overline\jsr[F]) + L_u'(w_2(v))v. 
\end{align*}
Thus,
\begin{align*}
\overline\jsr[F+\epsilon u ] - \overline\jsr[F] &=   \inf (\{v: L_F'(w_1)v + \epsilon L_u(\overline\jsr[F])+\epsilon L_u'(w_2)v = 0 \}) \\
&=  - \epsilon \frac{L_u(\overline\jsr[F])}{L_F^{'}(w_1)+\epsilon L_u^{'}(w_2)}. 
\end{align*}
As we are only concerned with infinitesimal $\epsilon$ and finite $L_u(\overline\jsr[F])$, the value of $w_1$ given by the mean-value theorem is forced to take the value of $\overline\jsr[F]$. Thus,
\[
\lim_{\epsilon \downarrow 0} \overline\jsr[F+\epsilon u ] - \overline\jsr[F] 
= - \lim_{\epsilon \downarrow 0} \epsilon \frac{L_u(\overline\jsr[F])}{L_F^{'}(\overline\jsr[F])}.
\]
The lemma follows from observing that the derivative of $L_F$ is always negative.
\end{proof}

\subsection{Proof of Theorem \ref{thm:main-theorem}}
\label{pf:main-theorem}

We denote $\min_{h \in \H} \jsr((Y-h(X))^2)$ as $\jsr^*$ and some $h \in \H$ that achieves this minimum as $h^*$, i.e.,
\[
h^* \in \argmin_{h \in \H} \jsr((Y-h(X))^2) = \argmin_{h \in \H} \E[\ell((Y-h(X))^2-\jsr ^*)].
\]
Let $U=\ell'(0)>0$ be the slope of $\ell$ at 0. Note that the Lipschitz constant $G$ of $\ell$ is greater than $U$. Let $c_1 = \frac{1}{U}$ and $c_2=\frac{G}{U}>1$. 

The core proof argument is to show that the interval $[\alpha_t, \beta_t]$ always intersects with the interval $[\jsr^* - (c_1\rho+c_2\rho'), \jsr^* + (c_1\rho+c_2\rho')]$ even as the interval $[\alpha_t, \beta_t]$ shrinks by a factor of $2$ every iteration. 
This is made precise in the lemma below.
\begin{lemma}
\label{lem:intervalintersect}
    For any $t>0$, the interval $[\alpha_t, \beta_t] \cap [\jsr^* \pm (c_1\rho+c_2\rho')]$ is non-empty.
\end{lemma}
\begin{proof}
This is true at initialization as $\alpha_0=0 \leq \jsr^*$ and $\beta_0=\srarg{n}(\z^0) \in [\jsr((Y-h_0(X))^2) \pm \rho']$ and $\jsr((Y-h_0(X))^2) \geq \jsr^*$. 

Let us assume $[\alpha_{t-1}, \beta_{t-1}]$ intersects with $[\jsr^* - (c_1\rho+c_2\rho'), \jsr^* + (c_1\rho+c_2\rho')]$ for some $t\geq 1$ for a proof by induction.

\textbf{Case 1:} $\gamma_t$ lies in $[\jsr^* - (c_1\rho+c_2\rho'), \jsr^* + (c_1\rho+c_2\rho')]$. 

Regardless of whether the first  half or second half of $[\alpha_{t-1}, \beta_{t-1}]$ is chosen as $[\alpha_{t}, \beta_{t}]$, it is guaranteed to intersect with $[\jsr^* - (c_1\rho+c_2\rho'), \jsr^* + (c_1\rho+c_2\rho')]$.

\textbf{Case 2:} $\gamma_t < \jsr^* - (c_1\rho+c_2\rho') $

Consider the decision variable $\gamma$ in step 9 of the Algorithm.
\begin{align*}
    \gamma &= \srarg{n}(\z^t) \\
    &\geq \jsr((Y-g_t(X)^2) - \rho' \\
    &\geq \jsr^* - \rho' \\
    &\geq \jsr^* - c_2 \rho' \\
    &\geq \jsr^* - c_1 \rho - c_2 \rho' \\
    &\geq \gamma_t
\end{align*}
Thus the new interval $[\alpha_{t},\beta_{t}] = [\gamma_t, \beta_{t-1}]$ also intersects with $[\jsr^* \pm c_1\rho+c_2\rho']$.

\textbf{Case 3:} $\gamma_t > \jsr^* + (c_1\rho+c_2\rho') $

Consider the decision variable $\gamma$ in step 9 of  Algorithm \ref{alg:bisection_UBSR}.
\begin{align}
    \gamma &= \srarg{n}(\z^t) \nonumber \\
    &\leq \jsr((Y-g_t(X)^2)) + \rho' \nonumber \\
    &= \inf(\{u: \E[\ell((Y-g_t(X))^2 - u)] \leq \lambda \}) + \rho'.   \label{eqn:main-thm-eq-1} 
\end{align}
Notice that
\begin{align}
\E\left[\ell((Y-g_t(X))^2-\gamma_t) \right] 
&\leq  \min_{g\in \H} \E\left[\ell((Y-g(X))^2-\gamma_t) \right] + \rho \nonumber \\
&\leq \min_{g\in \H} \E\left[\ell((Y-g(X))^2-\jsr^* - c_1 \rho -c_2\rho') \right] + \rho \nonumber \\
&\leq \min_{g\in \H} \E\left[\ell((Y-g(X))^2-\jsr^*)\right] - U c_1 \rho - U c_2\rho' + \rho \nonumber \\
&= \lambda - U c_1 \rho - U c_2\rho' + \rho \nonumber \\
&= \lambda - G \rho'.  \label{eqn:main-thm-eq-2} 
\end{align}
As $\ell$ is convex, $G$-Lipschitz and monotone increasing, with slope at $0$ given by $\ell'(0)=U$, we have for any $0<b<a$ that $\ell(a)-bG \leq \ell(a-b) \leq \ell(a) - bU$. 
\begin{align}
\E\left[\ell\left((Y-g_t(X))^2-\gamma_t + \rho'\right) \right] 
&\leq 
\E\left[\ell((Y-g_t(X))^2-\gamma_t) \right] + G \rho' \nonumber\\
&\leq \lambda.
\label{eqn:main-thm-eq-3} 
\end{align}
Therefore,
\[
\inf(\{u: \E[\ell((Y-g_t(X))^2 - u)] \leq \lambda \}) \leq \gamma_t - \rho'.
\]
Using the above bound in \eqref{eqn:main-thm-eq-1}, we obtain
\[
\gamma \leq \gamma_t.
\]
Thus, the new interval $[\alpha_{t},\beta_{t}] = [\alpha_{t-1}, \gamma_t]$ also intersects with $[\jsr^* - (c_1\rho+c_2\rho'), \jsr^* + (c_1\rho+c_2\rho')]$.

Hence, by induction the interval $[\alpha_t, \beta_t]$ with length $2^{-t}\beta_0$ contains a point that is within $c_1 \rho + c_2 \rho'$ of $\jsr^*$.     
\end{proof}

Now we turn our attention to proving the main theorem.
\begin{proof}\textit{\textbf{(Theorem \ref{thm:main-theorem})}}

Let $c_3 = \beta_0$.
From \Cref{lem:intervalintersect}, we have 
\begin{align}
    |\gamma_t - \jsr^* | \leq c_1 \rho + c_2 \rho' + c_3 2^{-t}.
\label{eqn:gamma_t_bound}
\end{align}
Notice that
\begin{align*}
&\E\left[\ell((Y-g_t(X))^2-\gamma_t) \right] \\
&\leq  \min_{g\in \H} \E\left[\ell((Y-g(X))^2-\gamma_t) \right] + \rho  \\    
&\leq  \min_{g\in \H} \E\left[\ell((Y-g(X))^2-(\jsr^* - c_1 \rho - c_2 \rho' - c_3 2^{-t} ) \right] + \rho   \\    
&\leq \min_{g\in \H} \E\left[\ell((Y-g(X))^2-\jsr^*) \right] + Gc_1 \rho + Gc_2 \rho' + Gc_3 2^{-t}  + \rho,  \numberthis\label{eq:s123}
\end{align*}
where the penultimate inequality used \eqref{eqn:gamma_t_bound} and the final inequality used Lipschitzness of $\ell$.

For any $0<b<a$, we have 
\begin{align}
\ell(a)-bG \leq \ell(a-b) \leq \ell(a) - bU.\label{eq:lconvex-lip}
\end{align}
Letting $\epsilon =\frac{1}{U}( (Gc_1+1) \rho + Gc_2 \rho' + Gc_3 2^{-t})$, we  have
\begin{align*}
\E\left[\ell((Y-g_t(X))^2-\gamma_t - \epsilon) \right] 
&\leq 
\E\left[\ell((Y-g_t(X))^2-\gamma_t)\right] - U\epsilon \\
&\leq 
\min_{g\in \H} \E\left[\ell((Y-g(X))^2-\jsr^*) \right] \\
&= \lambda.\numberthis\label{eq:gt12}
\end{align*}
In the above, the first inequality used \eqref{eq:lconvex-lip}, the second follows from \eqref{eq:s123}, and the  equality above follows by definition of $\jsr^*$.

Thus, we have
\begin{align*}
\jsr((Y-g_t(X))^2) 
&= 
\inf (\{u: \E [\ell((Y-g_t(X))^2 - u)] \leq \lambda \}) \\
&\leq 
\gamma_t + \epsilon \\
&= \gamma_t + \frac{1}{U}( (Gc_1+1) \rho + Gc_2 \rho' + Gc_3 2^{-t}) \\
&\leq \jsr^* + c_1 \rho + c_2\rho' + c_3 2^{-t} +\frac{1}{U}( (Gc_1+1) \rho + Gc_2 \rho' + Gc_3 2^{-t}),
\end{align*}
where the first inequality follows from \eqref{eq:gt12}, and the final inequality from \eqref{eqn:gamma_t_bound}.

Let $C_1 = c_1+\frac{1}{U} (Gc_1+1)$ and $C_2 = c_2 + \frac{1}{U}Gc_2$ and $C_3 = c_3 +\frac{1}{U}Gc_3$. We then have
\[
\jsr((Y-g_t(X))^2) 
\leq 
\jsr^* + C_1 \rho + C_2 \rho' + C_3 2^{-t}.
\]
The main claim follows.
\end{proof}

\subsection{Proof of Corollary  \ref{cor:linear-models-bound}}
\label{pf:cor-lin-models}

\begin{proof}
    The bound on $\rho'$ follows from \Cref{thm:main-theorem}. The bound on $\rho$ follows from standard Rademacher complexity of the bounded linear hypothesis class for a Lipshcitz continuous loss function, e.g. \cite{bartlett2002rademacher}. The main claim follows by invoking \Cref{thm:main-theorem} with the aforementioned bounds on $\rho, \rho'$, respectively.
\end{proof}

\section{Conclusions}
\label{sec:conclusions}
We considered a risk-sensitive variant of the expected risk minimization problem in supervised learning, with UBSR as the objective. UBSR is a convex risk measure that is desirable over the popular CVaR measure, and to the best of our knowledge, we are the first to optimize UBSR in a supervised learning context. In the context of a regression problem, we proposed and analyzed a bisection-type algorithm to minimize excess UBSR. An important component of our algorithm is the linear minimization oracle, which in turn utilizes an expression for the gradient of UBSR in the space of distributions. The latter expression and the oracle might be of independent interest. 
As future work, it would be interesting to perform a detailed empirical investigation of the bisection-type algorithm that we proposed for UBSR optimization, esp. on datasets where an UBSR-optimal solution is desirable.

\bibliography{refs}
\bibliographystyle{plainnat}


\end{document}